\documentclass{l4dc2024}

\title[Verification of Neural Reachable Tubes]{Verification of Neural Reachable Tubes via Scenario Optimization and Conformal Prediction}

\usepackage{graphicx,caption,subcaption}
\usepackage{times}
\usepackage{mathtools}
\usepackage[rightcaption]{sidecap}
\usepackage{wrapfig}

\coltauthor{\Name{Albert Lin} \Email{albert.k.lin@usc.edu}\\ 
 \Name{Somil Bansal} \Email{somilban@usc.edu}\\
 \addr Department of Electrical and Computer Engineering, University of Southern California, CA, USA}

\newcommand{\tvar}{t}
\newcommand{\tdummy}{\tau}
\newcommand{\R}{\mathbb{R}}
\newcommand{\ctrl}{u}

\newcommand{\cfunc}{u(\cdot)}

\newcommand{\cset}{\mathcal{U}}

\newcommand{\state}{x}
\newcommand{\traj}{\xi} 

\newcommand{\dyn}{f} 
\newcommand{\targetfunc}{l}
\newcommand{\targetset}{\mathcal{L}}
\newcommand{\safeset}{\mathcal{S}}

\newcommand{\BRT}{\text{BRT}}
\newcommand{\policy}{\pi}

\newcommand{\costfunctional}{J}

\newcommand{\vfunc}{V}

\newcommand{\trajstandard}{\traj_{\state,\tvar}^{\ctrl}}
\newcommand{\trajinittime}{\traj_{\state,0}^{\ctrl}}

\newcommand{\veh}{Q}

\newcommand{\horizon}{T}

\newcommand{\safetyMetric}{\delta_{\Tilde{\vfunc},\Tilde{\policy}}}
\newcommand{\costFunction}{J_{\Tilde{\policy}}}
\newcommand{\learnedCostFunction}{\Tilde{J}_{\Tilde{\policy}}}

\begin{document}

\maketitle

\begin{abstract}%
 Learning-based approaches for controlling safety-critical autonomous systems are rapidly growing in popularity; thus, it is important to provide rigorous and robust assurances on their performance and safety.
 Hamilton-Jacobi (HJ) reachability analysis is a popular formal verification tool for providing such guarantees, since it can handle general nonlinear system dynamics, bounded adversarial system disturbances, and state and input constraints.
 However, it involves solving a Partial Differential Equation (PDE), whose computational and memory complexity scales exponentially with respect to the state dimension, making its direct use on large-scale systems intractable.
 To overcome this challenge, neural approaches, such as DeepReach, have been used to synthesize reachable tubes and safety controllers for high-dimensional systems.
 However, verifying these neural reachable tubes remains challenging.
 In this work, we propose two different verification methods, based on robust scenario optimization and conformal prediction, to provide probabilistic safety guarantees for neural reachable tubes.
 Our methods allow a direct trade-off between resilience to outlier errors in the neural tube, which are inevitable in a learning-based approach, and the strength of the probabilistic safety guarantee.
 Furthermore, we show that split conformal prediction, a widely used method in the machine learning community for uncertainty quantification, reduces to a scenario-based approach, making the two methods equivalent not only for verification of neural reachable tubes but also more generally.
 To our knowledge, our proof is the first in the literature to show a strong relationship between the highly related but disparate fields of conformal prediction and scenario optimization.
 Finally, we propose an outlier-adjusted verification approach that harnesses information about the error distribution in neural reachable tubes to recover greater safe volumes.
 We demonstrate the efficacy of the proposed approaches for the high-dimensional problems of multi-vehicle collision avoidance and rocket landing with no-go zones.%
\end{abstract}

\begin{keywords}%
  Probabilistic Safety Guarantees, Safety-Critical Learning, Neural Certificates, Hamilton-Jacobi Reachability Analysis, Scenario Optimization, Conformal Prediction%
\end{keywords}

\section{Introduction}

It is important to design provably safe controllers for autonomous systems. 
Hamilton-Jacobi (HJ) reachability analysis provides a powerful framework to design such controllers for general nonlinear dynamical systems \citep{lygeros2004reachability, mitchell2005time}.
In reachability analysis, safety is characterized by the system's \textit{Backward Reachable Tube (BRT)}.
This is the set of states from which trajectories will eventually reach a given target set despite the best control effort.
Thus, if the target set represents undesirable states, the BRT represents unsafe states and should be avoided. 
Along with the BRT, reachability analysis provides a safety controller to keep the system outside the BRT. 

Traditionally, the BRT computation in HJ reachability is formulated as an optimal control problem.
The BRT can then be obtained as a sub-zero level solution of the corresponding value function.
Obtaining the value function requires solving a partial differential equation (PDE) over a state-space grid, resulting in an exponentially scaling computation complexity with the number of states \citep{bansal2017hamilton}.
To overcome this challenge, a variety of solutions have been proposed that trade off between the class of dynamics
they can handle, the approximation quality of the BRT,
and the required computation.
These include specialized methods for linear and affine dynamics \citep{10.1007/3-540-64358-3_38, Frehse2011, Kurzhanski00, Kurzhanski02, Maidens13, girard2005reachability, althoff2010computing, bak2019numerical, Nilsson2016}, polynomial dynamics \citep{doi:10.1177/0278364914528059, majumdar2017funnel, Dreossi16, henrion2014convex}, monotonic dynamics \citep{coogan2015efficient}, and convex dynamics \citep{chow2017algorithm}
(see \citet{bansal2017hamilton, bansal2021deepreach} for a survey).

Owing to the success of deep learning, there has also been a surge of interest in approximating high-dimensional BRTs \citep{rubies2019classification, fisac2019bridging, djeridane2006neural, niarchos2006neural, darbon2020overcoming} and optimal controllers \citep{onken2022neural} through deep neural networks (DNNs).
Building upon this line of work, \citet{bansal2021deepreach} have proposed DeepReach -- a toolbox that leverages recent advances in neural implicit representations and neural PDE solvers to compute a value function and a safety controller for high-dimensional systems.
Compared to the aforementioned methods, DeepReach can handle general nonlinear dynamics, the presence of exogenous disturbances, as well as state and input constraints during the BRT computation.
Consequently, methods for verifying neural reachable tubes have been proposed. For example, \citet{lin2023generating} propose an iterative scenario-based method \citep{campi2009scenario} to recover probabilistically safe reachable tubes from DeepReach solutions up to a desired confidence level and bound on violation rate. Unfortunately, the method does not allow an after-the-fact risk-return trade-off, and as a result, it is highly sensitive to outlier errors in the learned solutions. This can lead to highly conservative reachable tubes and a severe loss of recovery in the case of stringent safety requirements, as we demonstrate in our case studies.

In this work, we propose two different verification methods, one based on robust scenario optimization and the other based on conformal prediction, to provide probabilistic safety guarantees for neural reachable tubes.
Both methods are resilient to the outlier errors in neural reachable tubes and automatically trade-off the strength of the probabilistic safety guarantees based on the outlier rate.
The proposed methods can evaluate any candidate tube and are not restricted to a specific class of system dynamics or value functions.
We further prove that these seemingly different verification methods naturally reduce to one another, providing a unifying viewpoint for uncertainty quantification (typical use case of conformal prediction) and error optimization (typical use case of scenario optimization) in neural reachable tubes.
Based on these insights, we propose an outlier-adjusted verification approach that can recover a greater safe volume from a neural reachable tube by harnessing information about the distribution of error in the learned solution.
To summarize, the key contributions of this paper are:
\begin{itemize}
    \item probabilistic safety verification methods for neural reachable tubes that enable a direct trade-off between resilience and the probabilistic strength of safety,
    \item a proof that split conformal prediction reduces to a scenario-based approach in general, demonstrating a strong relationship between the two highly related but disparate fields,
    \item an outlier-adjusted verification approach that recovers greater safe volumes from tubes, and
    \item a demonstration of the proposed approaches for the high-dimensional problems of multi-vehicle collision avoidance and rocket landing with no-go zones.
\end{itemize}
\section{Problem Setup} \label{sec:problem_setup}
\vspace{-0.5em}
Consider a dynamical system with state $\state \in X \subseteq \R^n$, control $\ctrl \in \cset$, and dynamics $\dot{\state} = \dyn(\state, \ctrl)$ governing how $\state$ evolves over time until a final time $\horizon$. 
Let $\trajstandard(\tdummy)$ denote the state achieved at time $\tdummy \in [\tvar, \horizon]$ by starting at initial state $\state$ and time $\tvar$ and applying control $\cfunc$ over $[\tvar,\tau]$. 
Let $\targetset$ represent a target set that the agent wants to either reach (e.g. goal states) or avoid (e.g. obstacles).
\vspace{0.2em}

\noindent \textit{\textbf{Running example: Multi-Vehicle Collision Avoidance.}}
Consider a 9D multi-vehicle collision avoidance system with 3 independent Dubins3D cars: $\veh_1, \veh_2, \veh_3$. 
$\veh_i$ has position $(p_{xi}, p_{yi})$, heading $\theta_i$, constant velocity $v$, and steering control $u_i \in [u_{\min}, u_{\max}]$. 
The dynamics of $\veh_i$ are: $\dot{p}_{xi} = v\cos{\theta_i}, \quad \dot{p}_{yi} = v\sin{\theta_i}, \quad \dot{\theta_i} = u_i$.
$\targetset$ is the set of states where any of the vehicle pairs is in collision: $\targetset = \{x: \min\{d(\veh_1, \veh_2), d(\veh_1, \veh_3), d(\veh_2, \veh_3)\} \le R\}$, where $d(\veh_i, \veh_j)$ is the distance between $\veh_i$ and $\veh_j$.
We set: $v=0.6, \quad u_{\min}=-1.1, \quad u_{\max}=1.1, \quad R=0.25$.

In this setting, we are interested in computing the system's initial-time Backward Reachable Tube, which we denote as $\BRT$. 
We define $\BRT$ as the set of all initial states in $X$ from which the agent will eventually reach $\targetset$ within the time horizon $[0,\horizon]$, despite best control efforts: $\BRT = \{\state: \state\in X, \forall \ctrl(\cdot), \exists \tdummy \in [0, \horizon], \trajinittime(\tdummy) \in \targetset\}$.
When $\targetset$ represents unsafe states for the system, as it does in our running example, staying outside of $\BRT$ is desirable. 
When $\targetset$ instead represents the states that the agent wants to reach, $\BRT$ is defined as the set of all initial states in $X$ from which the agent, acting optimally, can eventually reach $\targetset$ within $[0,\horizon]$. Thus, staying within $\BRT$ is desirable.

The above 9D system is intractable for traditional grid-based methods, motivating the use of DeepReach to learn a neural $\BRT$.
Our goal in this work is to recover an approximation of the safe set with probabilistic guarantees.
Specifically, we want to find $\safeset$ such that $\underset{\state \in \safeset}{\mathbb{P}} ( \state \in \BRT ) \le \epsilon$ for some violation parameter $\epsilon \in (0, 1)$.
When $\targetset$ represents goal states, we want $\underset{\state \in \safeset}{\mathbb{P}} ( \state \in \BRT^C ) \le \epsilon$.
\vspace{-1.5em}
\section{Background: Hamilton-Jacobi Reachability, DeepReach, and Safety Verification} \label{sec:hj_reachability}
\vspace{-0.5em}
Here, we provide a quick overview of Hamilton-Jacobi reachability analysis, a specific toolbox, DeepReach, to compute high-dimensional neural reachable tubes, and an iterative scenario-based method for recovering probabilistically safe tubes from learning-based methods like DeepReach.

\vspace{-0.5em}
\subsection{Hamilton-Jacobi (HJ) Reachability} \label{sec:background_reachability}
\vspace{-0.5em}
In HJ reachability, computing $\BRT$ is formulated as an optimal control problem. We will explain it in the context of $\targetset$ being a set of undesirable states. In the end, we will comment on when $\targetset$ is a set of desirable states and refer interested readers to \citet{bansal2017hamilton} for other cases.

We first define a target function $\targetfunc(\state)$ such that the sub-zero level of $\targetfunc(\state)$ yields $\targetset$: $\targetset = \{\state: \targetfunc(\state) \le 0\}$.
$\targetfunc(\state)$ is commonly a signed distance function to $\targetset$. For example, we can choose $\targetfunc(\state) = \min\{d(\veh_1, \veh_2), d(\veh_1, \veh_3), d(\veh_2, \veh_3)\}-R$ for our running example in \sectionref{sec:problem_setup}.
Next, we define the cost function of a state corresponding to some policy $\cfunc$ to be the minimum of $\targetfunc(\state)$ over its trajectory: $\costfunctional_{\cfunc}(\state,\tvar) = \min_{\tdummy \in [\tvar, \horizon]} \targetfunc(\trajstandard(\tdummy))$.
Since the system wants to avoid $\targetset$, our goal is to maximize $\costfunctional_{\cfunc}(\state,\tvar)$.  
Thus, the value function corresponding to this optimal control problem is:
\vspace{-0.5em}
\begin{equation}
    \label{eq:valuefunc}
    \vfunc(\state,\tvar) = \sup_{\cfunc} \costfunctional_{\cfunc}(\state,\tvar)
    \vspace{-0.5em}
\end{equation}
By defining our optimal control problem in this way, we can recover $\BRT$ using the value function. In particular, the value function being sub-zero implies that the target function is sub-zero somewhere along the optimal trajectory, or in other words, that the system has reached $\targetset$. Thus, $\BRT$ is given as the sub-zero level set of the value function at the initial time: $\BRT = \{\state: \state\in X, \vfunc(\state,0) \le 0 \}$.
The value function in Equation \eqref{eq:valuefunc} can be computed using dynamic programming, resulting in the following final value Hamilton-Jacobi-Bellman Variational Inequality (HJB-VI): $\min\Big\{D_\tvar \vfunc(\state,\tvar)+ H(\state,\tvar), \targetfunc(\state)-\vfunc(\state,\tvar)\Big\} = 0$, with the terminal value function $\vfunc(\state,\horizon) = \targetfunc(\state)$. 
$D_\tvar$ and $\nabla$ represent the time and spatial gradients of $\vfunc$. 
$H$ is the Hamiltonian that encodes the role of dynamics and the optimal control: $H(\state,\tvar) = \max_\ctrl \langle \nabla \vfunc(\state,\tvar), \dyn(\state,\ctrl)\rangle$.
The value function in Equation \eqref{eq:valuefunc} induces the optimal safety controller: $u^*(\state,\tvar) = \underset{\ctrl}{\arg\max} \langle \nabla \vfunc(\state,\tvar), \dyn(\state,\ctrl)\rangle$.
Intuitively, the safety controller aligns the system dynamics in the direction of the value function's gradient, thus steering the system towards higher-value states, i.e., away from $\targetset$.

We have just explained the case where $\targetset$ represents a set of undesirable states. When the system instead wants to reach $\targetset$, an infimum is used instead of a supremum in Equation $\eqref{eq:valuefunc}$.
The control wants to reach $\targetset$, hence there is a minimum instead of a maximum in the Hamiltonian and optimal safety controller equations. See \citet{bansal2017hamilton} for details on other reachability cases.

Traditionally, the value function is computed by solving the HJB-VI over a discretized grid in the state space. Unfortunately, doing so involves computation whose memory and time complexity scales exponentially with respect to the system dimension, making these methods practically intractable for high-dimensional systems, such as those beyond 5D. Fortunately, a deep learning approach, DeepReach, has been proposed to enable HJ reachability for high-dimensional systems. 

\vspace{-1em}
\subsection{DeepReach and an Iterative Scenario-Based Probabilistic Safety Verification Method} \label{sec:deepreach_iterative_verification}
\vspace{-0.5em}
Instead of solving the HJB-VI over a grid, DeepReach \citep{bansal2021deepreach} learns a parameterized approximation of the value function using a sinusoidal deep neural network (DNN). Thus, memory and complexity requirements for training scale with the value function complexity rather than the grid resolution, allowing it to obtain BRTs for high-dimensional systems.
DeepReach trains the DNN via self-supervision on the HJB-VI itself. 
Ultimately, it takes as input a state $\state$ and time $\tvar$, and it outputs a learned value function $\Tilde{\vfunc}(\state,\tvar)$.
$\Tilde{\vfunc}(\state,\tvar)$ also induces a corresponding safe policy $\Tilde{\policy}(\state,\tvar)$, as well as a $\BRT$ (referred to as the neural reachable tube from hereon).

However, the neural reachable tube will only be as accurate as $\Tilde{\vfunc}(\state,\tvar)$. To obtain a provably safe $\BRT$, \citet{lin2023generating} propose a uniform value correction bound which is defined, for the avoid case, as the maximum learned value of an unsafe state under the induced policy: $\safetyMetric \coloneqq \max_{x\in X}\{\Tilde{\vfunc}(\state,0): \costFunction(\state,0) \le 0\}$.
The authors show that the super-$\safetyMetric$ level set of $\Tilde{\vfunc}(\state,0)$ is provably safe under the policy $\Tilde{\policy}(\state,\tvar)$.
They also propose an iterative scenario-based probabilistic verification method for computing an approximation of $\safetyMetric$ from finite random samples that satisfies a desired confidence level and violation rate.
However, the method is sensitive to outlier errors in the neural reachable tube and can result in very conservative safe sets.
Specifically, it does not provide safety assurances for safe sets with nonzero empirical safety violations.

In this work, we propose probabilistic safety verification methods that allow nonzero empirical safety violations at the cost of the probabilistic strength of safety.
This enables a direct trade-off between resilience to outlier errors and the strength of the safety guarantee.
\vspace{-0.25em}
\begin{remark}
    Although we work with DeepReach solutions in particular for our problem setup, our proposed approaches can verify any general $\Tilde{\vfunc}(\state,\tvar)$ and $\Tilde{\policy}(\state,\tvar)$, regardless of whether DeepReach, a numerical PDE solver, or some other tool is used to obtain them.
\end{remark}
\vspace{-1.5em}
\section{Robust Scenario-Based Probabilistic Safety Verification Method}\label{sec:scenario-based_method}
\vspace{-0.5em}
Here, we propose a robust scenario-based probabilistic safety verification method for neural reachable tubes.
The new method is a straightforward application of a scenario-based sampling-and-discarding approach to chance-constrained optimization problems, which quantifies the trade-off between feasibility and performance of the optimal solution based on finite samples \citep{campi2011sampling}.
First, we explain the method when $\targetset$ represents undesirable states. In the end, we comment on when $\targetset$ represents desirable states. 

\noindent \textit{\textbf{Procedures:}}
Let $\safeset \subseteq X$ be a neural safe set that is, in the avoid case, the \textit{complement} of the neural reachable tube being verified.
In our case, $\safeset$ is typically a super-$\delta$ level set of the learned value function $\Tilde{\vfunc}(\state,0)$.
Ideally, any super-$\delta$ level set of $\Tilde{\vfunc}(\state,0)$ for $\delta > 0$ should be a valid safe set; however, due to learning errors, that might not be true in practice.
To provide a probabilistic safety assurance for $\safeset$, we first sample $N$ independent and identically distributed (i.i.d.) states $x_{1:N}$ from $\safeset$ according to some probability distribution $\mathbb{P}$ over $\safeset$.
Since $\safeset$ is defined implicitly by $\Tilde{\vfunc}(\state,0)$, we use rejection sampling.
We next compute the costs $\costFunction(\state_i,0)$ for $i=1, 2, ..., N$ by rolling out the system trajectory from $x_i$ under $\Tilde{\policy}(x,t)$.
Let $k$ refer to the number of ``outliers'' - samples that are empirically unsafe, i.e., $\costFunction(\state_i,0) \le 0$.
Then the following theorem provides a probabilistic guarantee on the safety of the neural reachable tube and its complement, the neural safe set $\safeset$: 
%
\vspace{-0.5em}
\begin{theorem}[Robust Scenario-Based Probabilistic Safety Verification]
    \label{theorem:scenario-based_method}
    Select a safety violation parameter $\epsilon \in (0, 1)$ and a confidence parameter $\beta \in (0, 1)$ such that
    \vspace{-0.5em}
    \begin{equation} \label{eq:theorem_parameters_relationship}
        \begin{aligned}
            \sum^k_{i=0} \binom{N}{i} \epsilon^i (1-\epsilon)^{N-i} \le \beta
        \end{aligned}
    \end{equation}
    where $k$ and $N$ are as defined above. 
    Then, with probability at least $1-\beta$, the following holds:
    \begin{equation} \label{eq:robust_guarantee}
        \begin{aligned}
            \underset{\state \in \safeset}{\mathbb{P}} \left( \vfunc(\state,0) \le 0 \right) \le \epsilon
        \end{aligned}
    \end{equation}
\end{theorem}
All proofs can be found in the Appendix of the extended version of this article\footnote{\label{note:extended_paper}See \url{https://sia-lab-git.github.io/Verification_of_Neural_Reachable_Tubes.pdf}}.
Disregarding the confidence parameter $\beta$ for a moment, \theoremref{theorem:scenario-based_method} states that the fraction of $\safeset$ that is unsafe is bounded above by the violation parameter $\epsilon$, where $\epsilon$ is computed empirically using Equation \eqref{eq:theorem_parameters_relationship} based on the outlier rate $k$ encountered within $N$ samples.
$\epsilon$ is thus a reflection of the safety quality of $\safeset$, which degrades with the increase in the number of outliers $k$, as expected. This can also be seen for the running example in \figureref{fig:fixed_N} (the red curve).
Overall, \theoremref{theorem:scenario-based_method} allows us to compute probabilistic safety guarantees for any neural set $\safeset$ based on a finite number of samples.
Subsequently, this result can be used to find some $\safeset$ for which $\epsilon$ is smaller than a desired threshold, as we discuss later in this section.

To interpret $\beta$, note that $k$ is a random variable that depends on the randomly sampled $x_{1:N}$. It may be the case that we just happen to draw an unrepresentative sample, in which case the $\epsilon$ bound does not hold. $\beta$ controls the probability of this adverse event happening, which regards the correctness of the probabilistic safety guarantee in Equation \eqref{eq:robust_guarantee}. Fortunately, $\beta$ goes to $0$ exponentially with $N$, so $\beta$ can be chosen to be an extremely small value, such as $10^{-16}$, when we sample large $N$. $1-\beta$ will then be so close to $1$ that it does not have any practical importance.

We have just explained the robust scenario-based probabilistic safety verification method in the case where $\targetset$ represents undesirable states. When the system instead wants to reach $\targetset$, $\safeset$ will be a sublevel set instead of a superlevel set of the learned value function.
The cost inequality should be flipped when computing $k$, and the value inequality should be flipped in Equation \eqref{eq:robust_guarantee}. 

\subsection{Comparison of Robust and Iterative Scenario-Based Probabilistic Safety Verification}\label{sec:comparison}
\vspace{-0.5em}
The key difference between the proposed robust scenario-based method and the iterative scenario-based method discussed in \sectionref{sec:deepreach_iterative_verification} is that the former can handle nonzero empirical safety violations $k$.
This enables several crucial advantages that we demonstrate in \figureref{fig:fixed_N,fig:diff_N} for a solution learned by DeepReach on the multi-vehicle collision avoidance running example in \sectionref{sec:problem_setup}. We have fixed the confidence parameter $\beta = 10^{-16}$ to be so close to $0$ that it has no practical significance ($\beta$ plays the same role in both methods).
\vspace{-0.5em}
\begin{figure}[h!]
    \begin{minipage}[c]{0.32\textwidth}
       \includegraphics[trim={0 0 35cm 0},clip, width=\textwidth]{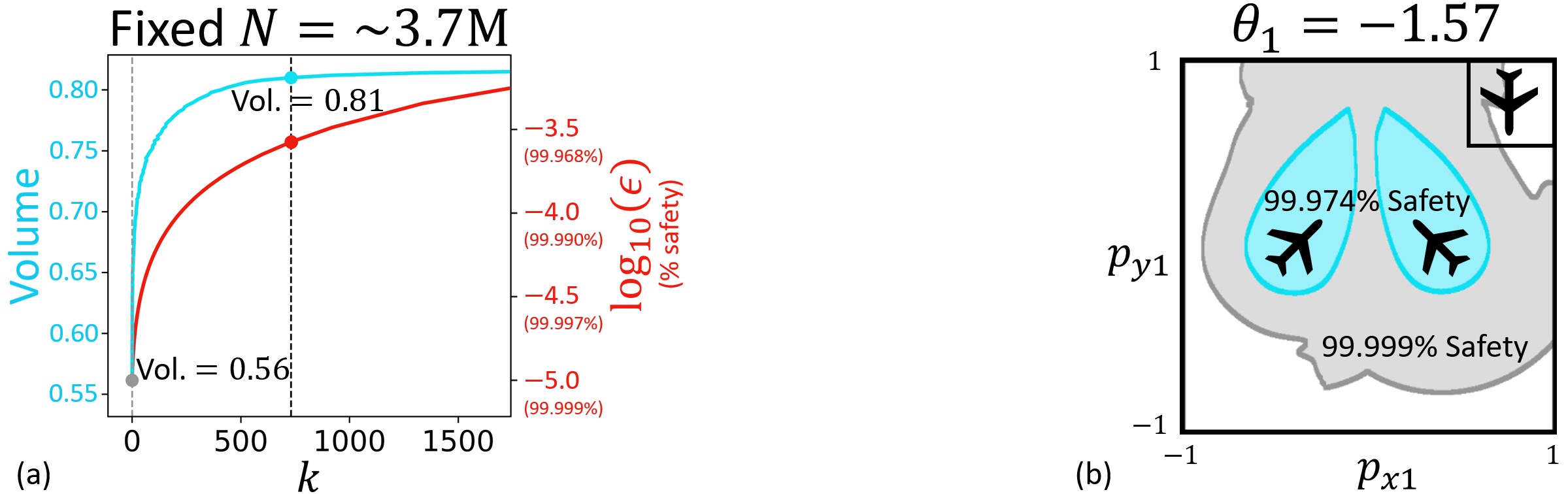}

        \vspace{0.5em}
       
        \includegraphics[trim={43cm 0 0 0},clip, width=0.75\textwidth]{figures/Figure_1}
    \end{minipage}\hfill
    \begin{minipage}[c]{0.67\textwidth}
        \caption{(Top) For a fixed simulation budget $N$, the cyan curve shows the number of empirical safety violations $k$ for different learned volumes (different super-levels of $\Tilde{\vfunc}(\state,0)$). The red curve shows the trade-off in safety strength $\epsilon$ (in log scale) for each $k$ using the robust method. The grey point indicates the iterative method baseline. The robust method is able to provide safety assurances even for the volumes that have non-zero outliers. (Dashed black line) By a small decrease in safety level (from $99.999\%$ to $99.974\%$) caused by outliers, we are able to significantly increase the assured safe volume from $0.56$ to $0.81$. (Bottom) Correspondingly, the safe set $\safeset$ increases greatly from the complement of the grey region to the complement of the blue region.} \label{fig:fixed_N}
    \end{minipage}
    \vspace{-1.0em}
\end{figure}

Firstly, for a fixed simulation budget $N$, the robust method allows one to trade off the probabilistic strength of safety (increasing $\epsilon$) for resilience (increasing $k$).
In other words, the method can verify any given neural safe set $\safeset$ in an outlier-robust fashion by automatically attenuating the level of safety assurance based on the number of empirical outliers (i.e., safety violations).
The iterative method, in contrast, can only verify a region that is outlier-free.
Consequently, the robust method enables one to engage in a trade-off if a large increase in safe set volume can be attained by a tolerable decrease in safety, as illustrated in \figureref{fig:fixed_N}.
\vspace{-0.5em}
\begin{figure}[h!]
    \centering
    \includegraphics[width=0.95\columnwidth]{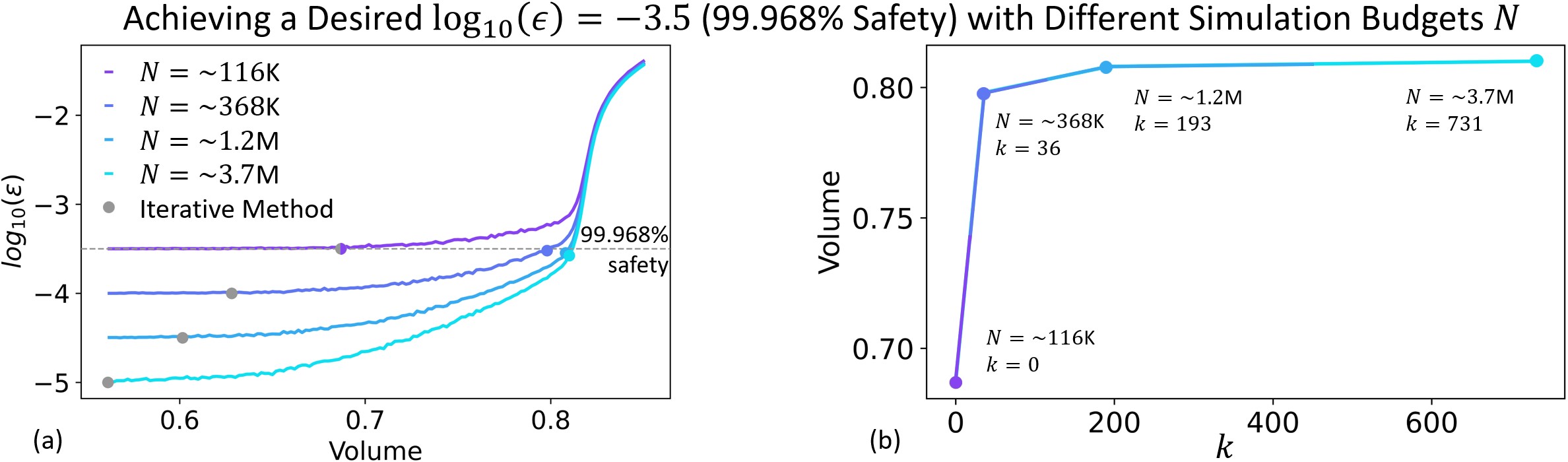}
    \vspace{-0.5em}
    \caption{(Left) Computing the safety strength $\epsilon$ across different volumes (different super-levels of $\Tilde{\vfunc}(\state,0)$) for different simulation budgets $N$ using the robust method. The grey points indicate the iterative method baselines. (Right) As we increase $N$, the largest volume achieving the desired $99.968\%$ safety using the robust method increases up to a limit.}
    \vspace{-1em}
    \label{fig:diff_N}
\end{figure}

Secondly, by allowing nonzero safety violations $k$, the robust method provides stronger safety assurances for a \textit{fixed} volume with increment in the simulation budget $N$, as long as the outlier rate does not grow substantially with $N$.
Thus, with more simulation effort, significantly larger volumes can be attained \textit{for a desired safety strength $\epsilon$} as shown in \figureref{fig:diff_N}.
Incrementing $N$ in the iterative method, on the other hand, will only correspond to verifying \textit{smaller} volumes at a \textit{stronger} $\epsilon$.
It cannot verify larger volumes for a fixed $\epsilon$, because empirical safety violations will be introduced.
\figureref{fig:diff_N} shows how the robust method (curves) adds a new degree of freedom for computing safety assurances compared to the iterative method (grey points).

\section{Conformal Probabilistic Safety Verification Method}\label{sec:conformal_method}
\vspace{-0.5em}
We now propose a \textit{conformal} probabilistic safety verification method for neural reachable tubes which is intended to be the direct analogue of the \textit{robust scenario-based} method in \sectionref{sec:scenario-based_method}.
The method is a straightforward application of split conformal prediction, a widely used method in the machine learning community for uncertainty quantification \citep{MAL-101}.

Using the same procedures as described in \sectionref{sec:scenario-based_method}, split conformal prediction can be used instead of robust scenario optimization to provide a probabilistic guarantee on the safety of the neural reachable tube and its complement, the neural safe set $\safeset$:
%
\begin{theorem}[Conformal Probabilistic Safety Verification]
    \label{theorem:conformal_method}
    Let the number of outliers $k$ and the number of samples $N$ be as defined in the procedures in \sectionref{sec:scenario-based_method}, then:
\begin{equation}\label{eq:conformal_method_theorem}
        \underset{\state \in \safeset}{\mathbb{P}} \left( \costFunction(\state,0) > 0 \right) \sim \mathrm{Beta}(N-k, k+1)
    \end{equation}
\end{theorem}
\theoremref{theorem:conformal_method} can be established via a straightforward application of conformal prediction with $-\costFunction(\state,0)$ as the scoring function. 
The proof is in the Appendix of the extended version of this article\textsuperscript{\ref{note:extended_paper}}. 
The above theorem states that the fraction of $\safeset$ that is safe is distributed according to the Beta distribution with shape parameters $N-k$ and $k+1$. 
Intuitively, the mass in the distribution shifts towards $0$ as $k$ increases for a fixed $N$, implying that it is more likely that a smaller fraction of $\safeset$ is safe, as expected.
For a fixed ratio $N:k$, $N$ controls how concentrated the mass is around the mean; i.e., for larger sample sizes $N$, we can more confidently determine the fraction of $\safeset$ that is safe.

%
To better understand \theoremref{theorem:conformal_method}, we show in \figureref{fig:beta} the Beta distribution of $\underset{\state \in \safeset}{\mathbb{P}} \left( \costFunction(\state,0) > 0 \right)$ for a solution learned by DeepReach on the multi-vehicle collision avoidance running example in \sectionref{sec:problem_setup}, for which $k=731$ outliers are found from $N=3684118$ samples.
%
%
\begin{figure}[h!]
    \begin{minipage}[c]{0.23\textwidth}
       \includegraphics[width=\textwidth]{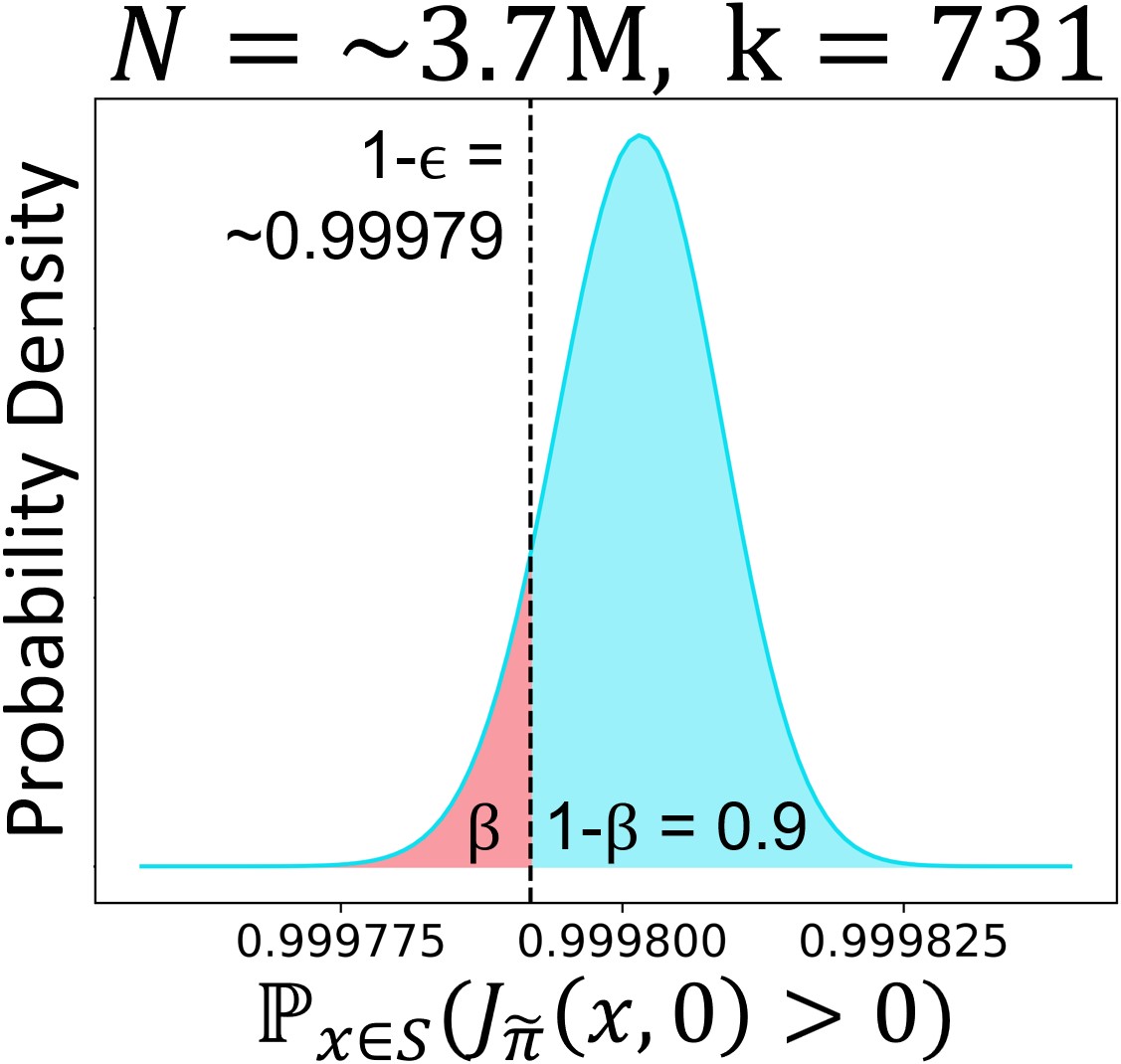}
    \end{minipage}\hfill
    \begin{minipage}[c]{0.73\textwidth}
        \caption{The Beta distribution of $\underset{\state \in \safeset}{\mathbb{P}} \left( \costFunction(\state,0) > 0 \right)$ when $k=731$ outliers are found from $N=3684118$ samples. (Dashed black line) For an example choice of confidence $1-\beta=0.9$ (shaded blue), we can lower-bound the fraction of $\safeset$ which is safe with at least $1-\epsilon=0.99979$ ($99.979\%$) confidence.} \label{fig:beta}
    \end{minipage}
    \vspace{-1.0em}
\end{figure}
\begin{remark} \label{remark:CP_coverage}
    The mean of the Beta distribution in Equation \eqref{eq:conformal_method_theorem} is given as $\frac{N-k}{N+1}$, which is roughly the fraction of the empirically safe samples. 
    One can immediately derive that the safety probability of $\safeset$, marginalized over the sampled ``calibration'' states, is given as: $\underset{ \left( \state_{1:N},\state \right) \in \safeset}{\mathbb{P}} \left( \costFunction(\state,0) > 0 \right) \ge \frac{N-k}{N+1}$, which precisely resembles the most commonly used \textbf{coverage property} of split conformal prediction.
\end{remark}

Even though \theoremref{theorem:conformal_method} provides the distribution of the safety level, when we compute safety assurances in practice, it is often desirable to know a lower-bound on the safety level with at least some desired confidence.
This corresponds to choosing a lower-bound whose accumulated probability mass is smaller than some confidence parameter $\beta$ (shaded red in \figureref{fig:beta}).
The following lemma formalizes this by using the CDF of the Beta distribution in \theoremref{theorem:conformal_method}. 
\begin{lemma}[Conformal Probabilistic Safety Verification]
    \label{lemma:conformal_method}
    Select a safety violation parameter $\epsilon \in (0, 1)$ and a confidence parameter $\beta \in (0, 1)$ such that
    \begin{equation} \label{eq:conformal_method_lemma}
        \begin{aligned}
            \sum^k_{i=0} \binom{N}{i} \epsilon^i (1-\epsilon)^{N-i} \le \beta
        \end{aligned}
    \end{equation}
    where $k$ and $N$ are as defined above. 
    Then, with probability at least $1-\beta$, the following holds:
    \begin{equation} \label{eq:conformal_method_lemma_guarantee}
        \begin{aligned}
            \underset{\state \in \safeset}{\mathbb{P}} \left( \vfunc(\state,0) \le 0 \right) \le \epsilon
        \end{aligned}
    \end{equation}
\end{lemma}
\lemmaref{lemma:conformal_method} is, in fact, precisely the same result as obtained by \theoremref{theorem:scenario-based_method} using robust scenario optimization.
This is no coincidence, as one can show that split conform prediction more generally reduces to a robust scenario-optimization problem.
\begin{remark}
    In general, a split conformal prediction problem can be reduced to a robust scenario-optimization problem. This is proven in the Appendix of the extended version of this article.\textsuperscript{\ref{note:extended_paper}}
\end{remark}
Due to the equivalence between conformal method and robust scenario-based methods, the analysis in \sectionref{sec:scenario-based_method} holds here as well.
More generally, we hope that this insight will lead to future research into further investigating the close relationship between the two methods.
\section{Outlier-Adjusted Probabilistic Safety Verification Approach}
\vspace{-0.5em}
The verification methods in \sectionref{sec:scenario-based_method,sec:conformal_method} are limited by the quality of the neural reachable tube.
Although they can account for outliers, the computed safety level can be low if the outlier rate is high.
This can lead to significant losses in the safe volume, as demonstrated in \sectionref{sec:rocketlanding,sec:reachavoidrocketlanding}.

To address this issue, we propose an outlier-adjusted approach that can recover a larger safe volume for any desired $\epsilon$.
Note that in the verification methods, the key quantity which determines $\epsilon$ is the number of safety violations $k$.
This corresponds to the number of samples $\state_i$ which are marked safe by membership in $\safeset$, i.e., $\Tilde{\vfunc}(\state_i,0) \ge \delta$, but are not guaranteed to be safe, i.e., $\costFunction\left(\state_i,0\right) \le 0$.
It is easy to see that the best we can do to simultaneously minimize $k$ and maximize volume is to compute $\safeset$ as the super-$\delta$ level set of the induced \textit{cost} function $\costFunction(\state,0)$.
For example, the largest possible $\safeset$ that is guaranteed to be violation-free is precisely the super-zero level set of $\costFunction(\state,0)$.
Thus, our overall approach will be to refine $\Tilde{\vfunc}(\state,0)$ so that it more accurately reflects $\costFunction(\state,0)$.

Modeling $\costFunction(\state,0)$ can be formulated as a supervised learning problem, since we can sample a state $\state_i$ and compute its cost $\costFunction(\state_i,0)$ in simulation.
We learn an approximation $\learnedCostFunction(\state,0)$ by retraining $\Tilde{\vfunc}(\state,0)$ on a training dataset $\mathcal{T}$ of $n$ samples, $\mathcal{T} = (\state_1, \costFunction(\state_1,0)), ..., (\state_n, \costFunction(\state_n,0))$.
Specifically, we use the \textit{weighted} MSE loss $\frac{1}{n}\sum_{i=1}^{n}w_{i}(\Tilde{\vfunc}(\state_i,0) - \costFunction(\state_i,0))^2$, where $w_{i} = w$ if the error is conservative $\left(\Tilde{\vfunc}(\state_i,0) < \costFunction(\state_i,0)\right)$, otherwise $w_i = 1$.
We introduce $w$ as a hyperparameter to underweight conservative errors because in the end, we are concerned with recovering larger \textit{safe} volumes.
Thus, selecting a small $w$ allows us to focus on reducing \textit{optimistic} errors $\left(\Tilde{\vfunc}(\state_i,0) > \costFunction(\state_i,0)\right)$ which are more safety-critical and correspond to outlier safety violations.

To avoid overfitting, we select the training checkpoint that performs best on a validation dataset $\mathcal{V}$.
The validation metric we use is the maximum learned cost of an empirically unsafe state: $\max_{x\in \mathcal{V}}\{\learnedCostFunction(\state,0): \costFunction(\state,0) \le 0\}$, which one can think of as a proxy for the recoverable safe volume.
We demonstrate the efficacy of the proposed outlier-adjusted approach for the high-dimensional systems of multi-vehicle collision avoidance and rocket landing with no-go zones. For all case studies, we set $w=10^{-3}$ during retraining, fix the confidence parameter $\beta = 10^{-16}$ and find a safe volume that satisfies $\epsilon \le 10^{-4}$ ($99.990\%$ safety) using the robust method in \sectionref{sec:scenario-based_method}.

\subsection{Multi-Vehicle Collision Avoidance} \label{sec:multicollision}
\vspace{-0.5em}
In \figureref{fig:multivehicle}, we compare our outlier-adjusted approach (blue) to the baseline (grey) for a DeepReach solution trained on the multi-vehicle collision avoidance running example in \sectionref{sec:problem_setup}.
A $2.3\%$ increase in the safe volume is attained, shown by the tightened $\BRT$. Note that the largest visual difference in the $\BRT$ is where the third vehicle is between the two others; intuitively, the safety in this region is likely more difficult to model by the baseline approach.
\vspace{-0.5em}
\begin{SCfigure}[1.22][htbp]
    \caption{Multi-Vehicle Collision Avoidance: outlier-adjusted (blue) and baseline (grey) results. (Left) Slice of the neural BRTs achieving $\epsilon=10^{-4}$ ($99.990\%$ safety). (Right) The outlier-adjusted approach increases the safe volume from $0.782$ to $0.8$ ($2.3\%$ increase).}
    \includegraphics[width=0.6\columnwidth]{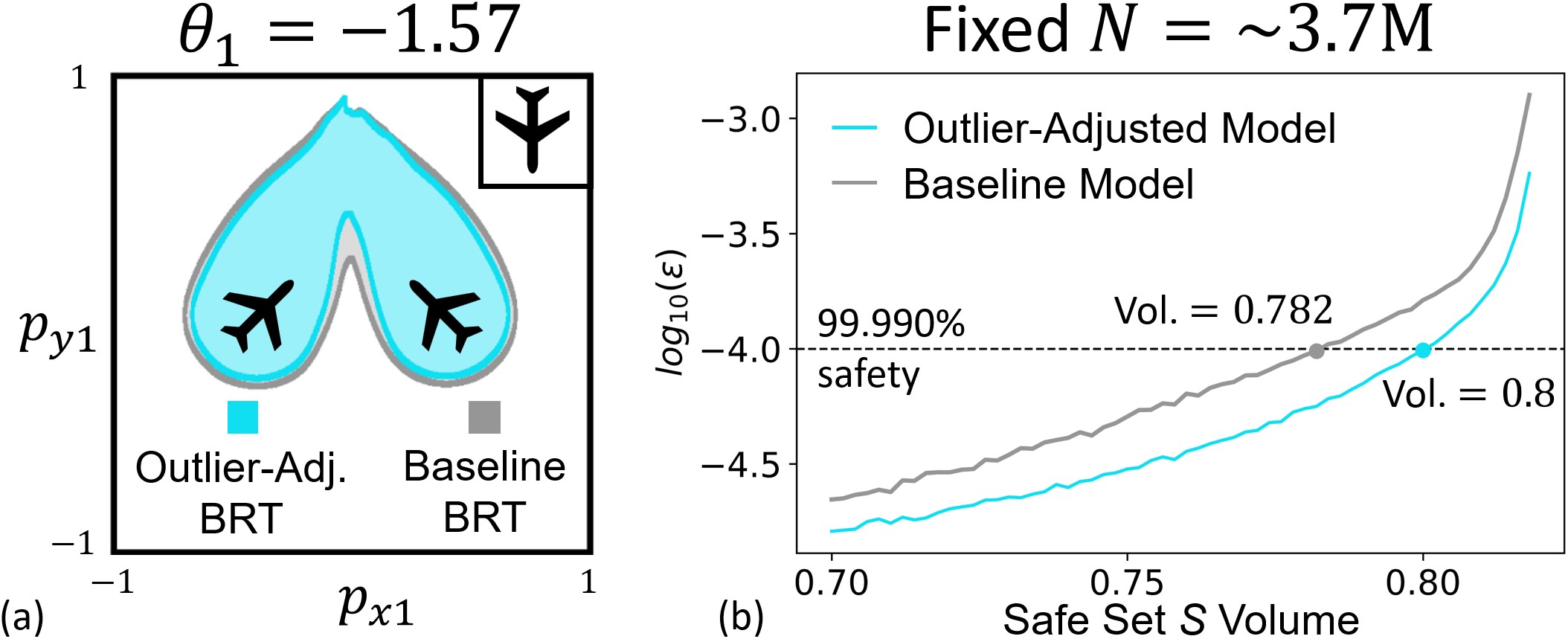}
    \label{fig:multivehicle}
\end{SCfigure}
\vspace{-1.0em}

\subsection{Rocket Landing} \label{sec:rocketlanding}
\vspace{-0.5em}
We now apply our approach to a 6D rocket landing system with position $(p_{x}, p_{y})$, heading $\theta$, velocity $(v_{x}, v_{y})$, angular velocity $\omega$, and torque controls $\tau_1, \tau_2\in [-250, 250]$. The dynamics are: $\dot{p_x} = v_x, ~\dot{p_y} = v_y, ~\dot{\theta} = \omega, ~\dot{\omega} = 0.3\tau_1, ~\dot{v_x} = \tau_1 \cos{\theta} - \tau_2 \sin{\theta}, ~\dot{v_y} = \tau_1 \sin{\theta} + \tau_2 \cos{\theta} - g$, where $g = 9.81$ is acceleration due to gravity.
The target set is the set of states where the rocket reaches a rectangular landing zone of side length 20m centered at the origin: $\targetset = \{\state: |p_x| < 20.0, p_y < 20.0\}$.
Note that we want to \textit{reach} $\targetset$, so the $\BRT$ now represents the safe set.
Results are shown in \figureref{fig:rocketlanding}.
Interestingly, a large $9.58\%$ increase in the volume of the safe set is recovered using the proposed approach, particularly near the lower-left part of the state space.
Further investigation reveals that the trajectories starting from these states exit the training regime south.
This highlights a general limitation of computing the value function over a constrained state space where information is propagated via dynamic programming, which affects both learning-based methods and traditional grid-based methods.
Nevertheless, in this case, the relative order of the value function levels is still preserved, leading to a high quality safe policy and recovery of a larger safe volume.

%
\begin{SCfigure}[1.22][htbp]
    \caption{Rocket Landing: outlier-adjusted (blue) and baseline (grey) results. (Left) Slice of the neural BRTs achieving $\epsilon=10^{-4}$ ($99.990\%$ safety). (Right) The outlier-adjusted approach increases the safe volume from $0.334$ to $0.366$ ($9.58\%$ increase).}
    \includegraphics[width=0.6\columnwidth]{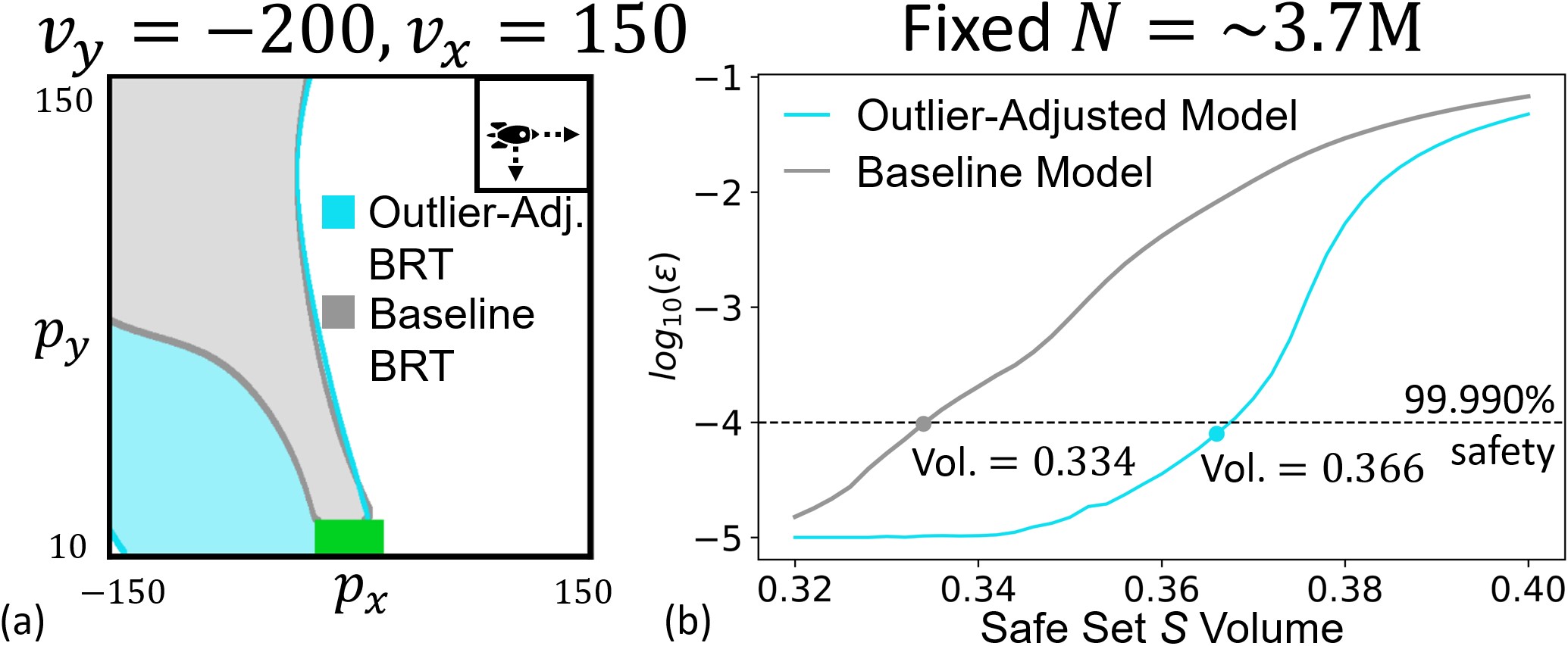}
    \label{fig:rocketlanding}
\end{SCfigure}
%

\subsection{Rocket Landing with No-Go Zones} \label{sec:reachavoidrocketlanding}
\vspace{-0.5em}
We now consider the rocket landing problem in a constrained airspace where we have no-go zones of height 100m and width 10m to the left of the landing zone and where altitude is below the landing zone.
Safety in this case takes the form of a reach-avoid set - the rocket needs to reach the landing zone while avoiding the no-go zones.
An analogous HJI-VI to the one in \sectionref{sec:background_reachability} can be derived for this case, whose solution can be computed using DeepReach.
However, since reach-avoid problems are more complex than just the reach or avoid problem, the DeepReach solution results in a poor safety volume.
In fact, \textit{no} safe volume can be recovered with the desired safety level of $\epsilon \le 10^{-4}$.
In contrast, we can recover a sizable safe volume using the outlier-adjusted approach, as shown in \figureref{fig:reachavoidrocketlanding}.
These examples highlight the utility of the proposed approach.

%
\begin{SCfigure}[1.22][htbp]
    \caption{Rocket Landing with No-Go Zones: outlier-adjusted (blue) and baseline (grey) results. (Left) Slice of the neural BRTs achieving $\epsilon=10^{-4}$ ($99.990\%$ safety). (Right) The outlier-adjusted approach increases the safe volume from $0$ to $0.19$.}
    \includegraphics[width=0.6\columnwidth]{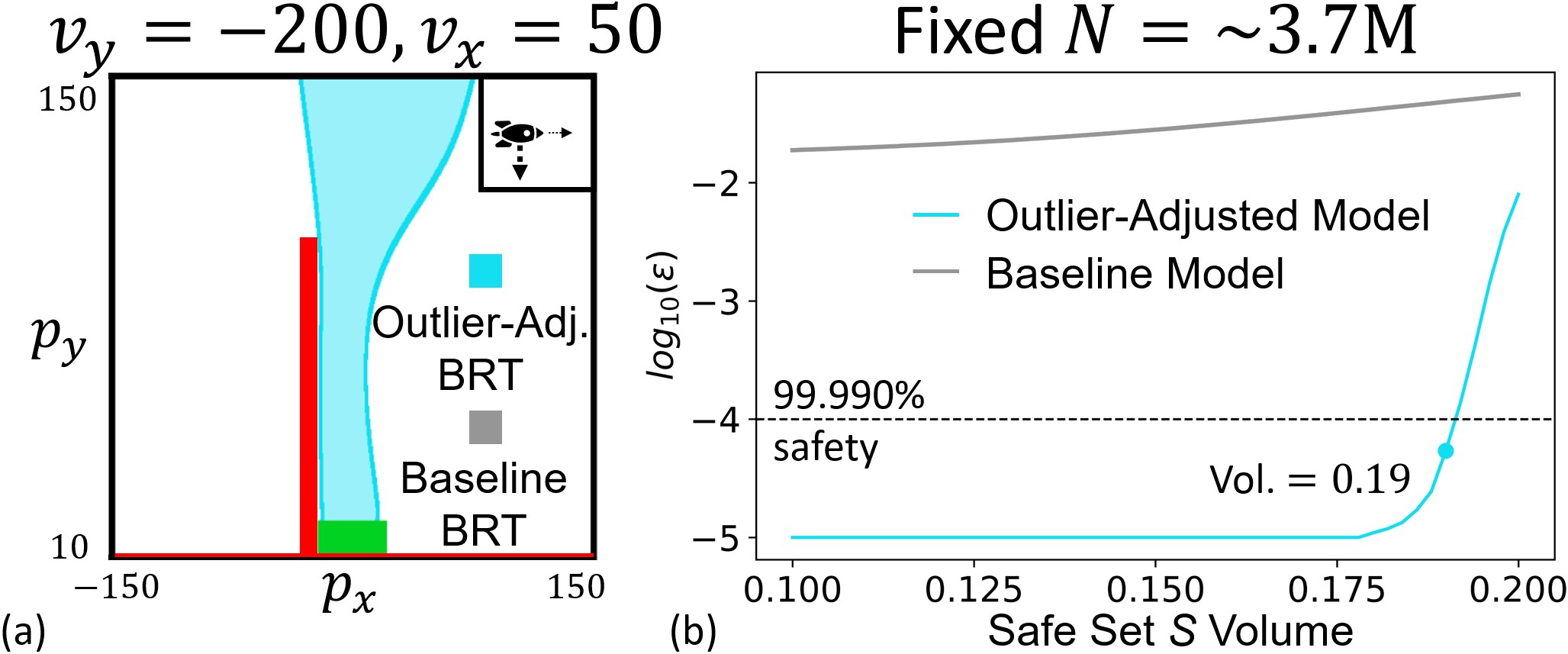}
    \label{fig:reachavoidrocketlanding}
\end{SCfigure}
%
\section{Discussion and Future Work}
\vspace{-0.5em}
In this work, we propose two different verification methods, based on robust scenario optimization and conformal prediction, to provide probabilistic safety guarantees for neural reachable tubes.
Our methods allow a direct trade-off between resilience to outlier errors in the neural tube, which are inevitable in a learning-based approach, and the strength of the probabilistic safety guarantee.
Furthermore, we show that split conformal prediction, a widely used method in the machine learning community for uncertainty quantification, reduces to a scenario-based approach, making the two methods equivalent not only for verification of neural reachable tubes but also more generally.
We hope that our proof will lead to future insights into the close relationship between the highly related but disparate fields of conformal prediction and scenario optimization.
Finally, we propose an outlier-adjusted verification approach that harnesses information about the error distribution in neural reachable tubes to recover greater safe volumes.
We demonstrate the efficacy of the proposed approaches for the high-dimensional problems of multi-vehicle collision avoidance and rocket landing with no-go zones.
Altogether, these are important steps toward using learning-based reachability methods to compute safety assurances for high-dimensional systems in the real world.

In the future, we will explore how the key idea of the outlier-adjusted verification approach, using cost labels as a supervised learning signal, can be used to enhance the accuracy of learning-based reachability methods like DeepReach.
Other directions include providing safety assurances in the presence of worst-case disturbances and in real-time for tubes that are generated online.

\acks{This work is supported in part by a NASA Space Technology Graduate Research Opportunity, the NVIDIA Academic Hardware Grant Program, the NSF CAREER Program under award 2240163, and the DARPA ANSR program.}

\bibliography{bibliography/bansal_papers, bibliography/formal_safety_references, bibliography/opt_ctrl_and_dp, bibliography/reachability, bibliography/references}

\newpage
\appendix
\section{Robust Scenario-Based Proofs}\label{apd:SO}

First, we introduce and prove the following lemma regarding a generic 1-dimensional chance-constrained optimization problem (CCP), which will be useful for subsequent proofs.
\begin{lemma}[Solution Feasibility for a 1-D CCP]\label{lemma:CCP}
    Consider the following 1-dimensional CCP:
    \begin{align}\label{eq:CCP}
    \begin{split}
        \text{CCP}_\epsilon: &\min_{g\in \mathbb{R}}{g} \\
        &\text{s.t. } \underset{h \in H}{\mathbb{P}} \left( f(h) \le g \right) \ge 1-\epsilon
    \end{split}
    \end{align}
    where $g$ is the 1-dimensional optimization variable, $h$ is the uncertain parameter that describes different instances of an uncertain optimization scenario, and $f$ is some function of $h$.
    
    The corresponding sample counterpart (SP) of this CCP is:
    \begin{align}\label{eq:SP}
    \begin{split}
        \text{SP}^A_{N,k}: &\min_{g \in \mathbb{R}}{g} \\
        &\text{s.t. } f(h_i) \le g, \quad i \in \{1, ..., N\}-A\{h_1, ..., h_N\}
    \end{split}
    \end{align}
    where $N$ constraints are sampled but $k$ constraints are discarded according to some constraint elimination algorithm $A$.
    Let $g^*_{N,k}$ denote the solution to the above SP. Select a violation parameter $\epsilon \in (0, 1)$ and a confidence parameter $\beta \in (0, 1)$ such that
    \begin{equation}\label{eq:parameters_relationship}
        \begin{aligned}
            \sum^k_{i=0} \binom{N}{i} \epsilon^i (1-\epsilon)^{N-i} \le \beta
        \end{aligned}
    \end{equation}
    Then, with probability at least $1-\beta$, the following holds:
    \begin{equation}\label{eq:CCP_guarantee}
        \begin{aligned}
            \underset{h \in H}{\mathbb{P}} \left( f(h) > g^*_{N,k} \right) \le \epsilon
        \end{aligned}
    \end{equation}
\end{lemma}
\begin{proof}
    \lemmaref{lemma:CCP} is a straightforward application of a scenario-based sampling-and-discarding approach to a CCP as detailed in \citet{campi2011sampling}.
    CCP \eqref{eq:CCP} satisfies the assumptions of \citet{campi2011sampling}, since both the domain of optimization $\mathbb{R}$ and the constraint sets parameterized by $h$, $\{g: f(h) \le g\}$, are convex and closed in $g$. Thus, \lemmaref{lemma:CCP} follows as a special case of Theorem 2.1 in \citet{campi2011sampling}, where $d=1$, $c=1$, $x=g$, $X=G=\mathbb{R}$, $\delta=h$, $\Delta=H$, and $X_\delta=G_h=\{g: f(h) \le g\}$.
\end{proof}

\subsection{Proof of \texorpdfstring{\theoremref{theorem:scenario-based_method}}{Robust Scenario-Based Probabilistic Safety Verification Theorem}}\label{apd:scenario-based_theorem_proof}

\begin{proof}
    Consider the chance-constrained optimization problem (CCP) \eqref{eq:CCP} in \lemmaref{lemma:CCP} directly above, where $h=x$, $H=\safeset$, and $f(h)=f(x)=-\costFunction(\state,0)$. The proposed robust scenario-based probabilistic safety verification method deals with the corresponding sample counterpart (SP) \eqref{eq:SP} in \lemmaref{lemma:CCP} where the constraint elimination algorithm $A$ is to remove all $k$ constraints $f(h_i) \le g$ where $f(h_i)=-\costFunction(\state_i,0) \ge 0$. Thus, the only constraints remaining are $f(h_i) \le g$ where $f(h_i) < 0$. Since we are minimizing $g$, the solution $g^*_{N,k}$ to the SP must be $< 0$; i.e., $0 < -g^*_{N,k}$. Therefore, $\underset{x \in \safeset}{\mathbb{P}} ( \costFunction(\state,0) \le 0 ) \le \underset{x \in \safeset}{\mathbb{P}} ( \costFunction(\state,0) < -g^*_{N,k} )$. Equation \eqref{eq:CCP_guarantee} of \lemmaref{lemma:CCP} then yields: $\underset{x \in \safeset}{\mathbb{P}} ( \costFunction(\state,0) < -g^*_{N,k} ) \le \epsilon \implies \underset{x \in \safeset}{\mathbb{P}} ( \costFunction(\state,0) \le 0 ) \le \epsilon$, where $\forall (\state,\tvar), \costFunction(\state,\tvar) \le V(\state, \tvar)$ from Equation \eqref{eq:valuefunc}, so Equation \eqref{eq:robust_guarantee} of \theoremref{theorem:scenario-based_method} directly follows.
\end{proof}
\section{Conformal Proofs}\label{sec:CP}

\subsection{Proof of \texorpdfstring{\theoremref{theorem:conformal_method}}{Conformal Probabilistic Safety Verification Theorem}}\label{apd:conformal_theorem_proof}

\begin{proof}

\theoremref{theorem:conformal_method} is a straightforward application of the split conformal prediction method detailed in \citet{MAL-101}, where we set the conformal ``input'' $x=x$, ``output'' $y=-\costFunction(x,0)$, ``score function'' $s(x,y)=y=-\costFunction(x,0)$, ``size of the calibration set'' $n=N$, and ``user-chosen error rate'' $\alpha=\frac{k+1}{N+1}$.
The conformal $\hat{q}$ is then computed as the $\frac{\lceil(N+1)(1-\alpha)\rceil}{N}$ quantile of the calibration scores $-\costFunction(x_{1:N},0)$.
The quantile is $\frac{\lceil(N+1)(1-\alpha)\rceil}{N} = \frac{\lceil(N+1)(1-\frac{k+1}{N+1})\rceil}{N} = \frac{\lceil(N+1)(\frac{N-k}{N+1})\rceil}{N} = \frac{N-k}{N}$, where we have defined $k$ in the procedures in \sectionref{sec:scenario-based_method} as the number of scores $-\costFunction(x_i,0) \ge 0$.
Thus, this quantile corresponds precisely to the largest \textit{negative} score, so we know that $\hat{q} < 0$.
Theorem 1 in \citet{MAL-101} then yields:
\begin{align}
    \underset{ \left( x_{1:N},x \right) \in \safeset }{\mathbb{P}} \left( -\costFunction(x,0) \le \hat{q}  \right) &\ge 1-\alpha \nonumber \\
    \underset{ \left( x_{1:N},x \right) \in \safeset }{\mathbb{P}} \left( \costFunction(x,0) \ge -\hat{q}  \right) &\ge 1-\frac{k+1}{N+1} \nonumber \\
    \underset{ \left( x_{1:N},x \right) \in \safeset }{\mathbb{P}} \left( \costFunction(x,0) > 0  \right) &\ge \frac{N-k}{N+1} \label{eq:cp_proof_coverage}
\end{align}
where Equation \eqref{eq:cp_proof_coverage} follows from the line preceding it because if $\costFunction(x,0) \ge -\hat{q}$, then certainly $\costFunction(x,0) > 0$.
This coverage property result is precisely the same as described in \remarkref{remark:CP_coverage}.
Furthermore, Section 3.2 in \citet{MAL-101} yields:
\begin{align}
    \underset{ x \in \safeset }{\mathbb{P}} \left( \costFunction(x,0) > 0  \right) &\sim \text{Beta}(N+1-l,l),\quad l= \lfloor (N+1)\alpha\rfloor \nonumber \\
    \underset{ x \in \safeset }{\mathbb{P}} \left( \costFunction(x,0) > 0  \right) &\sim \text{Beta}(N-k,k+1)
\end{align}
which is precisely the result of \theoremref{theorem:conformal_method}.
\end{proof}

\subsection{Proof that Split Conformal Prediction Reduces to Robust Scenario Optimization}\label{apd:reduction_proof}
Here, we show that split conformal prediction, in full generality, reduces to a robust scenario optimization problem.
We hope that this insight will encourage future research on the close relationship between the highly related but disparate fields of conformal prediction and scenario optimization.
In split conformal prediction, we first define a score function $s(x,y) \in \mathbb{R}$ which is meant to reflect the uncertainty for a model input $x$ and corresponding model output $y$.
Then, we sample an i.i.d. calibration set $(X_1,Y_1),...,(X_n,Y_n)$ and compute $\hat{q}$ as the $\frac{\lceil(n+1)(1-\alpha)\rceil}{n}$ quantile of the calibration scores $s(X_1,Y_1),...,s(X_n,Y_n)$, where $\alpha \in [0, 1]$ is a user-chosen error rate.
For a new i.i.d. sample $X_\text{test}$, we construct a prediction set $C(X_\text{test})=\{y: s(X_\text{test},y) \le \hat{q}\}$.
Theorem 1 in \citet{MAL-101} provides the following coverage property: $\mathbb{P} \left( Y_{\text{test}} \in C(X_\text{test}) \right) \ge 1-\alpha$.
This follows from the more powerful property, first introduced in \citet{pmlr-v25-vovk12}, which we prove reduces to a robust scenario-based result after:
\begin{equation}\label{eq:conditional_coverage}
    \mathbb{P} \left( Y_\text{test} \in C(X_\text{test}) | \{ (X_i, Y_i) \}^n_{i=1} \right) \sim \text{Beta}(n+1-l, l), \quad l=\lfloor (n+1)\alpha \rfloor
\end{equation}

\begin{proof}
    
To show that split conformal prediction reduces to a scenario-based approach, consider the CCP \eqref{eq:CCP} in \lemmaref{lemma:CCP} in \appendixref{apd:SO}, where $h=(x,y)$ and $f(h)=f\left((x,y)\right)=s(x,y)$.
That is, we want to find a probabilistic upper-bound on samples of the score function.
Then, consider the corresponding SP \eqref{eq:SP} in \lemmaref{lemma:CCP} where the $n$ sampled calibration scores $s(X_1, Y_1), ..., s(X_n, Y_n)$ forms our set of constraints.
Remove the $k=\lfloor (n+1)\alpha-1 \rfloor$ largest scores, where $\alpha$ is the user-chosen error rate in split conformal prediction.
The largest remaining score will be the $\frac{n-k}{n}$ quantile.
$\frac{n-k}{n}=\frac{n-\lfloor (n+1)\alpha-1 \rfloor}{n}=\frac{n+\lceil -((n+1)\alpha-1) \rceil}{n}=\frac{\lceil n -(n+1)\alpha + 1 \rceil}{n}=\frac{\lceil (n+1)(1-\alpha) \rceil}{n}$, which is precisely the same quantile as $\hat{q}$ in split conformal prediction. Thus, the solution to the SP is $g^*_{N,k}=\hat{q}$. \lemmaref{lemma:CCP} tells us that for a violation parameter $\epsilon \in (0, 1)$ and a confidence parameter $\beta \in (0, 1)$ that satisfies the relationship in Equation \eqref{eq:parameters_relationship}, with probability at least $1-\beta$, the following holds:
\begin{equation}\label{eq:scenario_based_conformal_prediction}
    \mathbb{P}\left( Y_\text{test} \in C(X_\text{test}) | \{ (X_i, Y_i) \}^n_{i=1} \right) \ge 1-\epsilon
\end{equation}

This is equivalent to Equation \eqref{eq:conditional_coverage}. To see why, note that the cumulative distribution function of the Beta distribution in Equation \eqref{eq:conditional_coverage} is given in terms of $k$ by the incomplete beta function ratio $I_x(n-k, k+1)=\sum^n_{j=n-k}\binom{n}{j}x^j(1-x)^{n-j}$ from \citet[\href{https://dlmf.nist.gov/8.17.E5}{(8.17.5)}]{NIST:DLMF}.
Changing the index $i=n-j$ yields $I_x(n-k, k+1)=\sum^k_{i=0}\binom{n}{n-i}x^{n-i}(1-x)^{n-(n-i)}=\sum^k_{i=0}\binom{n}{i}x^{n-i}(1-x)^{i}$.
Thus, Equation \eqref{eq:conditional_coverage} is equivalent to the claim that for any violation parameter $\epsilon \in (0, 1)$ and confidence parameter $\beta \in (0, 1)$, $\mathbb{P}\left( Y_\text{test} \in C(X_\text{test}) | \{ (X_i, Y_i) \}^n_{i=1} \right) \ge 1-\epsilon$ (Equation \eqref{eq:scenario_based_conformal_prediction}) holds with probability at least $1-\beta$ as long as $\beta \ge I_{1-\epsilon}(n-k, k+1)=\sum^k_{i=0}\binom{n}{i}\epsilon^{i}(1-\epsilon)^{n-i}$ (Equation \eqref{eq:parameters_relationship}).

\end{proof}

\subsection{Proof of \texorpdfstring{\lemmaref{lemma:conformal_method}}{Conformal Probabilistic Safety Verification Lemma}}\label{apd:conformal_lemma_proof}

\begin{proof}
    The conformal probabilistic safety verification method in \sectionref{sec:conformal_method} is nothing more than a specific formulation of the general split conformal prediction method in \appendixref{apd:reduction_proof}, which we have proven provides a result that is equivalent to the robust scenario optimization result in \lemmaref{lemma:CCP}.
    The result in \lemmaref{lemma:CCP}, when formulated in the context of the conformal probabilistic safety verification method in \sectionref{sec:conformal_method} and noting that $\forall (\state,\tvar), \costFunction(\state,\tvar) \le V(\state, \tvar)$ from Equation \eqref{eq:valuefunc}, is precisely \lemmaref{lemma:conformal_method}.
\end{proof}

\end{document}